\documentclass[11pt]{article}
\usepackage[letterpaper,margin=1in]{geometry}
\usepackage[T1]{fontenc}
\usepackage{lmodern}
\usepackage{amsmath,amssymb,amsthm}
\usepackage{microtype}
\usepackage{enumitem}
\usepackage[round,authoryear]{natbib}
\usepackage{xcolor}
\usepackage{hyperref}
\usepackage{needspace}
\definecolor{linkblue}{RGB}{30,65,100}
\hypersetup{colorlinks=true,linkcolor=linkblue,citecolor=linkblue,urlcolor=linkblue,
  pdftitle={An Order-Theoretic Characterization of Consistent Inductive Inference},
  pdfsubject={An order-theoretic characterization of fixed-target finite-mistake prediction}}
\newtheorem{theorem}{Theorem}[section]
\newtheorem{lemma}[theorem]{Lemma}
\newtheorem{proposition}[theorem]{Proposition}
\newcommand{\FH}{\mathcal F_H}
\newcommand{\Fh}{\mathcal F_h}
\newcommand{\tr}{\operatorname{tr}}
\newcommand{\graph}{\operatorname{graph}}
\newcommand{\dom}{\operatorname{dom}}
\newcommand{\KB}{\mathrm{KB}}

\setlist[enumerate]{leftmargin=*,itemsep=6pt,topsep=6pt}
\title{An Order-Theoretic Characterization of Consistent Inductive Inference}
\author{Zhou Lu\\ \texttt{leozoroaster@gmail.com}}
\date{September 2026}

\begin{document}
\maketitle
\vspace{-1.5em}
\begin{abstract}
When can a learner make only finitely many prediction errors along every
infinite sequence labeled by a fixed, unknown hypothesis? We characterize
this form of consistency for arbitrary binary hypothesis classes in ZFC,
without requiring a uniform mistake bound. The characterization uses a
single linear order on finite realizable traces. Each trace selects its
least subtrace, and the order must satisfy two conditions: conflicting
traces select different subtraces, and the order is well-founded on the
traces of each fixed target. These conditions induce a learner whose
selected evidence decreases on every mistake. Conversely, a consistent
learner yields such an order through canonical mistake transcripts and
the Kleene--Brouwer ordering. The result provides a representation of
consistent prediction by finite evidence, answering a question of \citet{lu2024}.
\end{abstract}

\section{Introduction}
The epistemological problem of induction concerns how experience can
justify claims about cases not yet observed: a regularity in past
observations does not by itself establish that the regularity will persist.
A learning-theoretic approach examines the long-run reliability of
inductive methods within a specified class of possible worlds
\citep{schulte1999,kelly2004}. We study a precise version of this question
for sequential prediction. Observations are labeled by one fixed, unknown
hypothesis, while inputs may be presented in an arbitrary order. Under
what conditions on the hypothesis class can a learner make only finitely
many prediction errors along every such sequence? Success means that,
on each sequence, every prediction is correct from some round onward.
The last mistaken round and the total number of mistakes may depend
on both the target and the entire input sequence.

This question depends on both the possible laws and the meaning of
success. A uniform finite mistake bound is characterized by finite
Littlestone dimension \citep{littlestone1988}. Allowing the bound to depend
on the target leads to the non-uniform theory of \citet{lu2024}.
Section~5 of that work isolates the weaker, sequence-dependent notion of
consistency studied here and asks for a characterization.\footnote{A different
protocol allows the hypothesis realizing the data to change with each
finite prefix, giving a stronger adversarial problem governed by
infinite Littlestone trees \citep{bousquet2021}. A characterization for
one fixed target must preserve this distinction.}

Our main result describes consistency through an order on finite
observation records. We call such a record a \emph{trace}: it contains
the observed labeled examples, with order and repetitions removed.
Given a linear order on traces, each trace selects its least subtrace.
The theorem requires that conflicting traces never select the same
subtrace and that no infinite descending sequence of traces be realized
by one target. The first condition makes selected evidence sufficient
to specify compatible predictions; the second prevents an endless
sequence of corrections. Together they are necessary and sufficient
for consistency.

The two directions explain different aspects of the characterization.
An order satisfying these conditions directly supplies a prediction
rule: use the compatible labels supplied by all traces with the currently
selected evidence. Every error forces a strict decrease in that
evidence. For the converse, an arbitrary consistent learner may depend
on the entire ordered history. We associate each finite trace with a
canonical history consisting only of mistakes. A replay property
recovers that history from its trace. Ordering the histories by the
Kleene--Brouwer construction, and breaking ties between traces, then
produces the required order. Consistency enters precisely when we
exclude infinite mistake histories for a fixed target.

The result is a structural characterization, rather than a numerical complexity measure or an effective decision procedure.
The characterization gives a normal form for consistent prediction.
Whenever any learner succeeds, there is one whose dependence on past
observations is entirely through an unordered trace and a subtrace
selected by one fixed order. Along every admissible sequence, both the
selected evidence and its decoded predictor eventually stabilize
(Proposition~\ref{prop:stabilize}). Section~\ref{sec:discussion} develops this
interpretation and relates it to mistake bounds, sample compression,
and inductive reasoning.

\paragraph{AI-assisted mathematical development.}
OpenAI's GPT models contributed substantially to the mathematical
development of this work, including the formulation of the
order-theoretic characterization and the construction of its proof.
The arguments were developed and refined through iterative interaction
with the author, who specified the problem, examined the
proposed proofs, and guided their refinement.

A Lean 4 formalization of the main characterization and stabilization proposition is available at \texttt{https://github.com/leozoroaster/finite-evidence-consistency-lean}.

\section{Setup and characterization}\label{sec:setup}
We work in ZFC.\footnote{The converse uses choice to well-order
$X\times2$. Choice is essential for the theorem over arbitrary domains:
if the converse held in ZF, apply it to the singleton class containing
the constant-zero function $h_0$. Its witness would well-order
$\mathcal F_{h_0}$, hence the singleton traces $\{(x,0)\}$, and
therefore $X$. Thus the universal converse implies the axiom of choice.}
Let $X$ be an arbitrary set, let $2=\{0,1\}$, and let
$H\subseteq 2^X$, where $2^X$ is the set of functions from $X$ to $2$.
Write $\omega=\{0,1,2,\ldots\}$. For a set $S$, write $S^{<\omega}$
for its finite sequences and $[S]^{<\omega}$ for its finite subsets.
Elements of $E=X\times2$ are \emph{labeled examples}; an element of
$E^{<\omega}$ is a \emph{history}.

\paragraph{Prediction protocol.}
Nature fixes an unknown target $h\in H$. At each round $n<\omega$, it
presents an input $x_n\in X$. The learner predicts a label in $2$ and
then observes $h(x_n)$. Inputs may repeat, and Nature may choose them
adaptively. For a deterministic learner, a guarantee for every input
sequence includes every sequence produced by such adaptive choices.

\paragraph{Consistency.}
A deterministic learner is a map $A:E^{<\omega}\times X\to2$.
We call $H$ \emph{consistent} if there exists a learner $A$ such that,
for every $h\in H$ and every $(x_n)_{n<\omega}\in X^\omega$,
\[
 \bigl|\{n<\omega:
 A(((x_i,h(x_i)))_{i<n},x_n)\ne h(x_n)\}\bigr|<\infty.
\]
Thus the total number of mistakes may depend on $h$ and on the entire
input sequence. No computability or measurability is required of $A$.
The term consistency refers throughout to this finite-mistake guarantee;
agreement with a finite sample will be stated explicitly.

\paragraph{Traces and compatibility.}
For any function $g:X\to2$, let
$\graph(g)=\{(x,g(x)):x\in X\}$. Define
\[
 \Fh=[\graph(h)]^{<\omega},\qquad
 \FH=\bigcup_{h\in H}\Fh.
\]
We call elements of $\FH$ \emph{realizable traces}, or simply
\emph{traces}. Each is a finite partial function from $X$ to $2$.
A history $\tau$ is realizable when its set of entries
$\tr(\tau)$ belongs to $\FH$; taking its trace discards order and
repetitions. Two traces $p,q$ \emph{conflict} if
$(x,0)\in p$ and $(x,1)\in q$, or conversely, for some $x\in X$.
Otherwise they are \emph{compatible}. Compatibility says that $p\cup q$
is a partial function; it does not require $p\cup q$ to be realizable
by a member of $H$.

\paragraph{Selected evidence.}
A strict linear order $\prec$ on $\FH$ is an irreflexive, transitive
relation that compares every two distinct traces. Write $\preceq$ for
its reflexive extension. For any such order and any $p\in\FH$, define
\begin{equation}\label{eq:minimum}
 m_\prec(p)=\min_\prec\{s:s\subseteq p\}.
\end{equation}
Every subtrace of $p$ belongs to $\FH$. The family in
\eqref{eq:minimum} is nonempty and finite, so its least element exists
uniquely, without any appeal to choice. This definition is conditional
on an order being given. It makes no assertion that a suitable order
exists. We refer to $m_\prec(p)$ as the \emph{selected evidence} of $p$;
``least'' always refers to $\prec$, not to cardinality.

\subsection{Characterization of consistency}
The characterization separates a local requirement on prediction from
a convergence requirement. Traces with the same selected evidence must
agree wherever they overlap. As observations accumulate, the selected
evidence can only move earlier in the order. The second condition below
ensures that such movement cannot continue indefinitely along a fixed
target.

\begin{theorem}[Finite-evidence characterization]\label{thm:main}
The class $H$ is consistent if and only if there exists a strict linear
order $\prec$ on $\FH$ satisfying the following conditions, with
$m=m_\prec$.
\begin{enumerate}
\item \textbf{Conflict separation.} For every $p,q\in\FH$,
\begin{equation}\label{eq:conflict}
 p,q\text{ conflict}\quad\Longrightarrow\quad m(p)\ne m(q).
\end{equation}
\item \textbf{Well-foundedness for each target.} For every $h\in H$,
the restriction $\prec|_{\Fh}$ is a well-order: each nonempty subset of
$\Fh$ has a least element. Equivalently, there are no $h\in H$ and
$(p_n)_{n<\omega}\in(\Fh)^\omega$ with
\begin{equation}\label{eq:descent}
 p_{n+1}\prec p_n\qquad(n<\omega).
\end{equation}
\end{enumerate}
\end{theorem}

The quantifiers are $\exists\prec\,\forall h\in H$: one order must
serve every target.\footnote{This is only a clarification of the scope
of the well-order requirement. The order need not be well-founded on
all of $\FH$. A global infinite descending sequence, if one exists,
must contain only finitely many terms in each $\Fh$. The existence
of such a sequence is neither assumed nor used in the proof.}
Section~\ref{sec:decoder} constructs the
predictor indexed by each selected subtrace, giving the term evidence
a precise operational meaning.

\section{Proof of the characterization}\label{sec:proof}
If $H=\varnothing$, any learner satisfies the consistency requirement
vacuously, and the empty relation satisfies the theorem. Henceforth
assume $H\ne\varnothing$. Write $\tau^\frown e$ for the history
obtained by appending the labeled example $e$ to $\tau$, and $()$ for
the empty history.

\subsection{Order implies consistency}\label{sec:decoder}
Let $\prec$ satisfy the two conditions of Theorem~\ref{thm:main},
and put $m=m_\prec$. We first construct a predictor from each possible
value of $m$. For $r\in m(\FH)$, let
\[
 B_r=\{p\in\FH:m(p)=r\},\qquad
 v_r=\bigcup_{p\in B_r}p.
\]
The family $B_r$ contains all realizable traces that select $r$ as
their evidence. By conflict separation, any two members of $B_r$ are
compatible. Consequently $v_r$ is a partial function: it never assigns
both labels to the same input. Extend it to a total predictor $g_r:X\to2$
by setting
\[
 g_r(x)=
 \begin{cases}
 b,&(x,b)\in v_r,\\
 0,&x\notin\dom(v_r).
 \end{cases}
\]
In particular, $g_r$ agrees with every trace in $B_r$.\footnote{The
predictor $g_r$ need not belong to $H$: compatibility of the traces in
$B_r$ does not require their union to be realizable by one hypothesis.}

Define
\[
 A_\prec(\tau,x)=
 \begin{cases}
 g_{m(\tr(\tau))}(x),&\tr(\tau)\in\FH,\\
 0,&\tr(\tau)\notin\FH.
 \end{cases}
\]
Thus the learner uses the accumulated trace to select a finite subtrace,
then predicts with its associated $g_r$. The union defining $g_r$ ranges
over all traces in $B_r$, rather than over the history currently observed.
It is a fixed part of the construction once $\prec$ and $H$ are given.

For $p,q\in\FH$, enlarging the trace enlarges the family over which
the minimum is taken. Therefore
\begin{equation}\label{eq:monotone}
 p\subseteq q\quad\Longrightarrow\quad m(q)\preceq m(p).
\end{equation}
Moreover, if $p\subseteq q$ and $(x,b)\in q$ satisfies
$g_{m(p)}(x)\ne b$, then $m(q)\prec m(p)$. Indeed, equality would
place $q$ in $B_{m(p)}$, so $(x,b)\in v_{m(p)}$ and
$g_{m(p)}(x)=b$, a contradiction. Thus a prediction error forces the
selected evidence to change strictly.

Fix a target $h\in H$ and a sequence $(x_n)_{n<\omega}\in X^\omega$.
Put
\[
 p_n=\{(x_i,h(x_i)):i<n\},\qquad r_n=m(p_n).
\]
Then $r_n\in\Fh$, and \eqref{eq:monotone} gives
$r_{n+1}\preceq r_n$ at every round. The preceding argument gives
$r_{n+1}\prec r_n$ whenever $A_\prec$ makes a mistake at round $n$.
If there were infinitely many mistakes, enumerate their indices as
$n_0<n_1<\cdots$. For each $k$,
\[
 r_{n_{k+1}}\preceq r_{n_k+1}\prec r_{n_k}.
\]
This would be an infinite descending sequence in $\Fh$, contrary to
\eqref{eq:descent}. Hence $A_\prec$ is consistent.

\subsection{Consistency implies an order}\label{sec:converse}
We now construct the order from a consistent learner. For each trace
$p$, we select a canonical history of mistakes whose terminal predictor
agrees with all of $p$. The trace of this history will become the least
subtrace of $p$ in our order. Replay and comparison lemmas establish
this claim and yield conflict separation. We then use consistency to
prove well-foundedness separately for each target.

\paragraph{Normalization.}
We may assume that the learner remembers observed labels and never makes
a mistake on a previously seen input. This simple convention ensures
that the mistake histories constructed below contain only distinct inputs.
Formally, let $A_0$ be a consistent learner. On a realizable history $\tau$,
define $A(\tau,x)$ to be the previously observed label when $x$ has
already appeared, and $A_0(\tau,x)$ otherwise. On unrealizable histories,
set $A=A_0$. Along any sequence labeled by a fixed $h\in H$, the two learners
receive the same histories, and every prediction on which they differ
is correct for $A$. Thus $A$ makes a subset of the mistakes of $A_0$
and remains consistent. With $f_\tau(x)=A(\tau,x)$, we have
\[
 \tr(\tau)\subseteq\graph(f_\tau)
 \qquad\text{for every realizable history }\tau.
\]
\paragraph{Canonical transcripts.}
By the axiom of choice, fix a well-order $<_E$ on $E=X\times2$,
to be used for every trace.
For a trace $p$ and a history $\tau$, let
\[
 D_p(\tau)=\{(x,b)\in p:f_\tau(x)\ne b\}
\]
be the examples in $p$ that the current predictor misclassifies.
Starting from $\tau_0=()$, recursively set
\[
 \tau_{k+1}=\tau_k^\frown\min_{<_E}D_p(\tau_k)
 \qquad\text{if }D_p(\tau_k)\ne\varnothing,
\]
and stop when this error set is empty. The procedure presents the
learner only with examples on which its current prediction is wrong.
The order $<_E$ makes the selection deterministic whenever several
such examples are available.

Every intermediate history has trace contained in $p$, and is therefore
realizable. If $(x,b)\in p$ has an input already seen in $\tau_k$,
its label agrees with the earlier observation because $p$ is a partial
function. Normalization then gives $f_{\tau_k}(x)=b$. Hence no such
example lies in $D_p(\tau_k)$: every appended input is new, and the
procedure stops after at most $|p|$ steps.

Denote its terminal history by $T(p)$, the \emph{canonical transcript}
of $p$, and let $r(p)=\tr(T(p))$. By construction and termination,
\[
 r(p)\subseteq p,\qquad p\subseteq\graph(f_{T(p)}).
\]
The transcript may omit many examples, but its terminal predictor agrees
with the whole trace. Since the learner depends on an ordered history,
the set $r(p)$ alone does not a priori record the information needed to
recover that predictor. The following lemma shows that the canonical
selection rule recovers the entire transcript from $r(p)$, and that
deleting any unselected examples leaves the transcript unchanged.

\Needspace{10\baselineskip}
\begin{lemma}[Replay]\label{lem:replay}
For $p,q\in\FH$,
\begin{equation}\label{eq:replay}
 r(p)\subseteq q\subseteq p\quad\Longrightarrow\quad T(q)=T(p).
\end{equation}
\end{lemma}
\begin{proof}
Both runs start at the empty history. Suppose inductively that they
have reached the same history $\tau$, and that the run on $p$ has not
stopped. Its next example is $e=\min_{<_E}D_p(\tau)$. Since $e$ appears
in $T(p)$, it belongs to $r(p)\subseteq q$. Both runs use the same
predictor $f_\tau$, so
\[
 e\in D_q(\tau)\subseteq D_p(\tau).
\]
The least element of $D_p(\tau)$ is therefore also the least element
of $D_q(\tau)$. Both runs append $e$, completing the induction step.
In particular, the run on $q$ cannot stop earlier. When the run on $p$
stops, $D_q(T(p))\subseteq D_p(T(p))=\varnothing$, so the run on $q$
stops at the same history.
\end{proof}

Taking $q=r(p)$ and then taking traces gives
\begin{equation}\label{eq:replay-idempotent}
 T(r(p))=T(p),\qquad r(r(p))=r(p).
\end{equation}
Next consider what happens when $p\subseteq q$. Until their runs
diverge, every error available in $p$ is also available in $q$.
The larger trace can introduce a smaller selected example at the first
divergence, or it can prolong the transcript after the run on $p$ has
stopped. We use an order on histories in which either change moves the
transcript earlier.

Equip $E^{<\omega}$ with the strict \emph{Kleene--Brouwer order}
$<_{\KB}$: $\sigma<_{\KB}\tau$ if $\tau$ is a proper prefix of
$\sigma$, or if neither is a prefix of the other and
$\sigma_j<_E\tau_j$ at their first differing position $j$.
Thus proper extensions precede their prefixes; otherwise the first
different entry determines the comparison. This is a strict linear
order. Write $\leq_{\KB}$ for its reflexive extension.

\begin{lemma}[Comparison]\label{lem:comparison}
For $p,q\in\FH$,
\begin{equation}\label{eq:comparison}
 p\subseteq q\quad\Longrightarrow\quad T(q)\leq_{\KB}T(p).
\end{equation}
\end{lemma}
\begin{proof}
First, $T(q)$ cannot be a proper prefix of $T(p)$. Otherwise, at the
history $T(q)$, the next example selected by the run on $p$ would be
a misclassified member of $p\subseteq q$, contradicting termination
of the run on $q$.

If $T(p)$ is a prefix of $T(q)$, including equality, the conclusion
follows from the extension-first convention. Otherwise let $\tau$ be
their longest common prefix and let $e,f$ be the respective next entries
of $T(p),T(q)$. At this history,
\[
 e=\min_{<_E}D_p(\tau)\in D_p(\tau)\subseteq D_q(\tau),
 \qquad f=\min_{<_E}D_q(\tau).
\]
Hence $f\leq_E e$. Since these entries differ, $f<_E e$, and the
definition of the Kleene--Brouwer order gives $T(q)<_{\KB}T(p)$.
Later entries cannot change this comparison.
\end{proof}

\paragraph{Constructing the order on traces.}
We compare traces primarily by their transcripts. Different traces may
have the same transcript, so we also need a tie-break. Define
$\triangleleft$ on $\FH$ by comparing cardinalities first, and then
comparing the $<_E$-increasing enumerations lexicographically among
traces of the same cardinality.

This is a well-order. For a fixed cardinality $n$, a nonempty family of
increasing enumerations has a lexicographically least member: choose
the least occurring first coordinate, then the least second coordinate
among those with that first coordinate, and continue for $n$ steps.
A nonempty family of finite traces therefore has a least member by first
choosing its least occurring cardinality. Moreover,
$s\subsetneq p$ implies $s\triangleleft p$. This last property is
the reason for using cardinality as the first tie-break: the transcript's
own trace must precede every larger trace having the same transcript.

Define $\prec$ on $\FH$ by
\[
 p\prec q\quad\Longleftrightarrow\quad
 T(p)<_{\KB}T(q)\quad\text{or}\quad
 \bigl[T(p)=T(q)\text{ and }p\triangleleft q\bigr].
\]
This is the lexicographic order on the pairs $(T(p),p)$.
Both constituent comparisons are strict linear orders, and the second
distinguishes traces with equal transcripts. Thus $\prec$ is a strict
linear order on $\FH$.

We claim that its least subtrace of $p$ is exactly $r(p)$.
For any $s\subseteq p$, replay and comparison give
\[
 T(r(p))=T(p)\leq_{\KB}T(s).
\]
If the inequality is strict, the primary comparison gives
$r(p)\prec s$. If the transcripts are equal, taking traces gives
$r(p)=r(s)\subseteq s$. Either $r(p)=s$, or the inclusion is proper,
in which case $r(p)\triangleleft s$. The tie-break therefore gives
$r(p)\preceq s$ in this case as well. Since $r(p)\subseteq p$, we
have proved
\begin{equation}\label{eq:minimum-trace}
 m_\prec(p)=r(p)\qquad(p\in\FH).
\end{equation}

\paragraph{Conflict separation.}
Suppose $m_\prec(p)=m_\prec(q)=s$. By
\eqref{eq:minimum-trace}, $r(p)=r(q)=s$, and replay yields
\[
 T(p)=T(r(p))=T(s)=T(r(q))=T(q).
\]
The terminal predictor $f_{T(s)}$ agrees with every example of $p$ and
of $q$. They cannot assign opposite labels to a common input. Equal
minima therefore imply compatibility, proving \eqref{eq:conflict}.

\paragraph{Well-foundedness for each target.}
The finite construction and conflict separation are now established.
Termination on each individual finite trace does not yet control an
infinite sequence of different traces. To obtain that control, we use
consistency on a fixed target and the following classical tree-ordering
fact \citep[Definition~4.5 and Lemma~4.6]{aschenbrenner2004}.
This fact turns the absence of infinite branches into a well-order on
the corresponding finite histories.

\begin{lemma}[Kleene--Brouwer]\label{lem:KB}
Let $W\subseteq E^{<\omega}$ be prefix-closed, with $E$ well-ordered
by $<_E$. Suppose $W$ has no infinite branch, meaning that no
$w\in E^\omega$ has every finite prefix in $W$. Then the restriction
of $<_{\KB}$ to $W$ is a well-order.
\end{lemma}
\begin{proof}
Assume $W\ne\varnothing$, since the empty case is immediate.
For $\tau\in W$, let $W_\tau$ contain all words in $W$ extending
$\tau$, including $\tau$ itself. We prove that $W_\tau$ is well-ordered
by first assuming this for the subtrees rooted at one-step extensions
of $\tau$. This is well-founded induction with extensions placed below
their prefixes. It is valid because, in ZFC, failure of well-foundedness would yield an infinite chain of proper extensions; its union would be an infinite word whose finite prefixes all lie in \(W\), contrary to the hypothesis.

Assume the assertion for the child subtrees $W_{\tau^\frown e}$ with
$\tau^\frown e\in W$. Each proper extension of $\tau$ belongs to
exactly one such subtree. If $e<_E f$, every word in
$W_{\tau^\frown e}$ precedes every word in $W_{\tau^\frown f}$,
because their first differing entries are $e$ and $f$. All words in
these child subtrees precede $\tau$ itself.

Take a nonempty $S\subseteq W_\tau$. If a child subtree meets $S$,
choose the $<_E$-least entry $e$ for which $W_{\tau^\frown e}$ meets
$S$, and then the least element of $S$ in that subtree, which exists
by induction. The preceding comparisons
show that this element is least in all of $S$. If no child subtree
meets $S$, then $S=\{\tau\}$. Thus $W_\tau$ is well-ordered.
Applying the induction to the empty word proves the claim.
\end{proof}

Fix $h\in H$ and consider its \emph{mistake tree}
\[
 \begin{split}
 W_h=\bigl\{((x_i,h(x_i)))_{i<n}:\;&n<\omega,\ (x_i)_{i<n}\in X^n,\\
 &A(((x_j,h(x_j)))_{j<i},x_i)\ne h(x_i)
       \text{ for every }i<n\bigr\}.
 \end{split}
\]
This is the tree of histories labeled by $h$ on which $A$ errs at
every round. To verify that it is prefix-closed, take a history
$\tau=((x_i,h(x_i)))_{i<n}\in W_h$ and any $k\le n$.
Retain its first $k$ examples and discard the last $n-k$, obtaining
$\sigma=((x_i,h(x_i)))_{i<k}$. At each retained round $i<k$,
the learner receives exactly the same preceding history
$((x_j,h(x_j)))_{j<i}$ and current input $x_i$ as in $\tau$.
Its prediction is therefore unchanged, so it still makes a mistake
at that round. Thus $\sigma$ is also a history labeled by $h$ on
which $A$ errs at every round, and hence $\sigma\in W_h$.
Every prefix of $\tau$, including the empty history when $k=0$,
therefore belongs to $W_h$, which is precisely prefix-closure.

Normalization ensures that its histories have distinct inputs. An infinite branch
would therefore give one infinite input sequence, labeled by this same
$h$, on which $A$ makes a mistake at every round. Consistency rules
out such a branch. By Lemma~\ref{lem:KB}, $<_{\KB}$ well-orders $W_h$.
The argument permits arbitrarily long finite branches and does not
require a uniform finite bound on their lengths.

For every $p\in\Fh$, the entries of $T(p)$ lie in
$p\subseteq\graph(h)$, and each is a mistake after the preceding
entries. Thus $T(p)\in W_h$. To prove well-ordering of the traces
themselves, take a nonempty $S\subseteq\Fh$. The set
$\{T(p):p\in S\}$ has a $<_{\KB}$-least transcript $\tau$.
Among the traces $p\in S$ with $T(p)=\tau$, choose the
$\triangleleft$-least one, denoted $p_*$. Any $q\in S$ with a
different transcript follows $p_*$ by the primary comparison; any
distinct $q\in S$ with the same transcript follows $p_*$ by the
tie-break. Hence $p_*$ is the $\prec$-least element of $S$.

It follows that $\prec|_{\Fh}$ is a well-order. Since $h$ was
arbitrary and $\prec$ was constructed before fixing $h$, the same
order satisfies the condition for every target. This completes the
proof of Theorem~\ref{thm:main}.\qed

\section{Discussion}\label{sec:discussion}
\subsection{Technical interpretation}
For fixed $H$ and a witness order, the selected evidence $m(p)$ is a
finite code for a predictor agreeing with the entire trace $p$:
\[
 p\subseteq\graph(g_{m(p)})\qquad(p\in\FH).
\]
The code is interpreted through the fixed decoder $r\mapsto g_r$;
it need not identify a particular member of $H$ or have minimum
cardinality.

Selection is stable under deletion of unselected examples:
\begin{equation}\label{eq:stability}
 m(p)\subseteq q\subseteq p\quad\Longrightarrow\quad m(q)=m(p).
\end{equation}
Indeed, $m(p)$ remains an available subtrace of $q$, while every
subtrace of $q$ was already available under $p$. Such a deletion
therefore changes neither the selected evidence nor its predictor.
Moreover, along any fixed target, the selected evidence eventually
stabilizes, giving a stronger conclusion than finitely many mistakes.

\Needspace{12\baselineskip}
\begin{proposition}[Stabilization of evidence]\label{prop:stabilize}
Fix a witness order, a target $h\in H$, and an input sequence. With
$p_n=\{(x_i,h(x_i)):i<n\}$, the sequence $m(p_n)$ is eventually
constant. Its eventual value $r_*$ satisfies
\[
 \bigcup_{n<\omega}p_n\subseteq\graph(g_{r_*}).
\]
\end{proposition}
\begin{proof}
The values $m(p_n)$ form a nonincreasing sequence in the well-order
$\prec|_{\Fh}$. Infinitely many changes would yield an infinite
strictly descending subsequence, so the values stabilize, say from
round $N$ onward. For each $n\geq N$, $p_n\in B_{r_*}$, and hence
$p_n\subseteq\graph(g_{r_*})$. Every observed example belongs to
such a $p_n$, giving the inclusion.
\end{proof}

Thus consistency always admits a learner that uses its history only
through the accumulated trace and whose predictor eventually stabilizes.
The limiting predictor agrees with the target on every input encountered
in the interaction, though it may disagree on inputs never presented.

There is also a finite code for each complete target under the same
decoder. Its existence gives a simple consequence on countable domains.

\Needspace{10\baselineskip}
\begin{proposition}[Target codes and countable domains]\label{prop:target-codes}
Fix a witness order. For each $h\in H$, let
$r_h=\min_\prec\Fh$. Then $r_h\in m(\FH)$ and $g_{r_h}=h$.
Consequently, on a countable domain, $H$ is consistent if and only if
$H$ is countable.
\end{proposition}
\begin{proof}
If $r_h\subseteq p\in\Fh$, then $r_h$ is an available subtrace of
$p$ and precedes every other such subtrace, so $m(p)=r_h$.
Applying this first to $p=r_h$ and then to
$p=r_h\cup\{(x,h(x))\}$ gives $r_h\in m(\FH)$ and
$g_{r_h}(x)=h(x)$ for every $x\in X$.
Thus $h\mapsto r_h$ is injective. When $X$ is countable, there are
only countably many finite traces, so consistency implies that $H$
is countable. Conversely, enumerate a nonempty countable $H$ and
predict using the first hypothesis agreeing with the history, or $0$
if none agrees. For a fixed target, every mistake eliminates a
hypothesis preceding the target's first occurrence in the enumeration; there
are only finitely many such hypotheses. The empty class is immediate.
\end{proof}

Decoding a finite target code uses the fixed order: other hypotheses
may agree with its labels.\footnote{Here ``finite'' counts labeled
examples. On an uncountable domain, it does not assert that their
inputs admit finite binary encodings.}
Moreover, an input sequence need not present all of $r_h$, so the
learner need not converge to $h$ on unobserved inputs.

\paragraph{Example: finitely many positive inputs.}
Let $H_{\mathrm{fin}}\subseteq2^\omega$ consist of the functions
with finitely many positive inputs, and put
$p^+=\{(x,1)\in p\}$ for each realizable trace $p$.
Fix a well-order $\triangleleft$ on $\mathcal F_{H_{\mathrm{fin}}}$.
Define $\prec$ by comparing traces first by decreasing $|p^+|$, then
by increasing $|p|$, and finally by $\triangleleft$.
Among the subtraces of $p$, the trace
$p^+$ retains all its positive examples and has the smallest possible
cardinality subject to doing so. Hence
\[
 m(p)=p^+.
\]
Conflicting traces have different positive parts, so conflict separation
holds. For a fixed target $h$, the value $|p^+|$ is bounded by the
finite number of positive inputs of $h$. The first comparison therefore
has only finitely many possible values on $\Fh$. Within each such
value, cardinality and $\triangleleft$ give a well-order. Thus
$\prec|_{\Fh}$ is a well-order.
The decoder $g_r$ predicts $1$ exactly at the positive inputs retained
in $r$. The induced learner remembers observed positive inputs and
predicts $0$ elsewhere; each mistake reveals a new positive input
of the target.

For comparison, consider the different protocol in which each finite prefix need only be realized by some hypothesis, possibly a different one for each prefix. Every finite prefix of \((0,1),(1,1),(2,1),\ldots\) is realized by a member of \(H_{\mathrm{fin}}\), but the entire sequence is not. It is therefore permitted by that alternative protocol and excluded from ours. This distinction explains why our theorem imposes well-foundedness on the traces of each fixed target.

\subsection{Connections to learning theory}
\paragraph{Descent after mistakes.}
Our learner shares the central idea of Littlestone's Standard Optimal
Algorithm (SOA): every mistake forces a decrease in an ordered state.
SOA predicts so that a mistake reduces the Littlestone dimension of the
hypotheses agreeing with the history, yielding a uniform mistake bound
when this dimension is finite \citep{littlestone1988}. Here the state
is a least subtrace, decoded into a predictor, and descent is controlled
by well-foundedness for each target rather than by a finite dimension.

\citet[Section~3]{bousquet2021} extend the descent argument using ordinal
Littlestone dimension. Their online protocol requires only realizability
of each finite prefix. Our condition instead reflects the requirement
of one common target: the same order serves all targets, but it need
be well-founded only on each target's traces.

\paragraph{Consistency and hypothesis-wise guarantees.}
Theorem~\ref{thm:main} resolves the characterization problem for
consistency raised in Section~5 of \citet{lu2024}. The main guarantee
studied there allows a mistake bound depending on $h$, but requires
it to hold uniformly over input sequences. This hypothesis-wise
requirement sits between the classical uniform bound, independent of
both target and sequence, and consistency, which requires only finitely
many mistakes for each individual target--sequence pair.

The distinction is strict, as the following example shows.

\Needspace{10\baselineskip}
\begin{proposition}[Consistency without a hypothesis-wise bound]\label{prop:separation}
Let $\omega_1$ be the first uncountable ordinal, let $X=\omega_1$, and
let
\[
 H=\{h_\alpha:\alpha<\omega_1\},
 \qquad h_\alpha(\xi)=\mathbf 1\{\xi<\alpha\}.
\]
Then $H$ is consistent, but no learner has a finite mistake bound
depending only on the target.
\end{proposition}
\begin{proof}
Start with $u=\omega_1$ and predict $\mathbf 1\{\xi<u\}$ on input
$\xi$. After a mistake, replace $u$ by $\xi$. For a fixed target
$h_\alpha$, the invariant $u\ge\alpha$ implies that a mistake can
occur only at an input satisfying $\alpha\le\xi<u$. Every mistake
therefore strictly decreases the ordinal $u$. There is no infinite
strictly decreasing sequence of ordinals, so the learner is consistent.

Every infinite subclass of $H$ has infinite Littlestone dimension.
Indeed, for any $d$, choose $2^d$ thresholds with distinct parameters.
At a node with $2m$ remaining parameters in increasing order, query
the $m$th parameter: exactly $m$ thresholds label it $0$ and $m$
label it $1$. Recursing on both groups gives a shattered binary tree
of depth $d$. Hence every subclass of finite Littlestone dimension
is finite, and the uncountable class $H$ cannot be a countable union
of such subclasses. The characterization in
\citet[Theorem~9]{lu2024} rules out a hypothesis-wise finite mistake
bound.
\end{proof}

\paragraph{Compression and reconstruction.}
The maps \(p\mapsto m(p)\) and \(r\mapsto g_r\) give a labeled compression--reconstruction representation, with no uniform bound on the retained sample size.
Equation~\eqref{eq:stability}
connects it to stable compression \citep{hanneke2021}, while the use
of a preference order recalls order compression schemes
\citep{darnstaedt2013}. Those schemes order a reconstruction class;
here the order is on finite traces, and conflict separation makes the
decoder well-defined. In the converse, canonical mistake transcripts and replay
recover a predictor from an unordered subtrace. The Kleene--Brouwer
order converts the absence of infinite mistake histories for each target
into the required well-foundedness.

\paragraph{Computability.}
The countability criterion in Proposition~\ref{prop:target-codes}
has an effective counterpart on $X=\omega$. A total computable
consistent learner exists if and only if there is a total computable
function $F:\omega\times\omega\to2$ such that
\[
 H\subseteq\{F(i,\cdot):i<\omega\}.
\]
For necessity, every $h\in H$ must equal $A(\tau,\cdot)$ at some
$h$-realizable finite history $\tau$: otherwise, repeatedly choosing
an input on which the current predictor disagrees with $h$ gives
infinitely many mistakes against that fixed target. Effectively
enumerating all finite histories therefore supplies the required
family. For sufficiency, at a history of length $n$, follow the least
index $i\le n$ whose predictor agrees with the history, or predict
$0$ if none does. This finite search defines a total computable
learner even on unrealizable histories. For target $F(k,\cdot)$,
there are at most $k$ mistakes before round $k$; thereafter, every
mistake permanently eliminates an index below $k$, so the total is
at most $2k$.
This is the fixed-target analogue of the projection and enumeration
argument of \citet[Lemma~3 and Theorem~5]{kalocinski2025}, whose
protocol requires realizability of every finite prefix.

\Needspace{8\baselineskip}
\subsection{Implications for inductive reasoning}
The theorem gives a means--ends characterization in the sense of
\citet{schulte1999}: for a specified class of possible targets and the
goal of eventually avoiding all prediction errors, it identifies exactly
when that goal is attainable. The guarantee is
conditional on the target belonging to $H$; choosing the admissible
possibilities remains a separate epistemic commitment.

\citet{kelly2004} connects Ockham's razor with efficiency in avoiding
retractions. Our result establishes that a preference over finite
evidence suffices for consistent prediction whenever consistency is
possible. It does not determine a notion of simplicity or minimize
revisions. An Ockham interpretation would therefore require additional
assumptions relating the witness order to simplicity and the desired
standard of efficiency.


\begin{thebibliography}{99}
\bibitem[Aschenbrenner and Pong(2004)]{aschenbrenner2004}
Matthias Aschenbrenner and Wai Yan Pong.
\newblock Orderings of monomial ideals.
\newblock \emph{Fundamenta Mathematicae}, 181(1):27--74, 2004.
\newblock \href{https://www.math.ucla.edu/~matthias/pdf/rd2.pdf}{Author manuscript}.

\bibitem[Bousquet et~al.(2021)]{bousquet2021}
Olivier Bousquet, Steve Hanneke, Shay Moran, Ramon van Handel, and Amir Yehudayoff.
\newblock A theory of universal learning.
\newblock In \emph{Proceedings of the 53rd Annual ACM SIGACT Symposium on
Theory of Computing}, pages 532--541, 2021.
\newblock \href{https://arxiv.org/abs/2011.04483}{Full version, arXiv:2011.04483}.

\bibitem[Darnst\"adt et~al.(2013)]{darnstaedt2013}
Malte Darnst\"adt, Thorsten Doliwa, Hans Ulrich Simon, and Sandra Zilles.
\newblock Order compression schemes.
\newblock In \emph{Algorithmic Learning Theory}, pages 173--187, 2013.
\newblock \href{https://www2.cs.uregina.ca/~zilles/darnstaedtDSZ13.pdf}{Author manuscript}.

\bibitem[Hanneke and Kontorovich(2021)]{hanneke2021}
Steve Hanneke and Aryeh Kontorovich.
\newblock Stable sample compression schemes: New applications and an optimal
SVM margin bound.
\newblock In \emph{Algorithmic Learning Theory}, volume 132 of
\emph{Proceedings of Machine Learning Research}, pages 697--721, 2021.
\newblock \href{https://proceedings.mlr.press/v132/hanneke21a.html}{Proceedings entry}.

\bibitem[Kaloci\'nski and Steifer(2025)]{kalocinski2025}
Dariusz Kaloci\'nski and Tomasz Steifer.
\newblock Computable universal online learning.
\newblock arXiv preprint arXiv:2510.18352, 2025.
\newblock \href{https://arxiv.org/abs/2510.18352}{Preprint}.

\bibitem[Kelly(2004)]{kelly2004}
Kevin T. Kelly.
\newblock Justification as truth-finding efficiency: How Ockham's razor works.
\newblock \emph{Minds and Machines}, 14:485--505, 2004.
\newblock \href{https://www.cmu.edu/dietrich/philosophy/docs/kelly/korbkluwerproofs.pdf}{Author proof}.

\bibitem[Littlestone(1988)]{littlestone1988}
Nick Littlestone.
\newblock Learning quickly when irrelevant attributes abound: A new
linear-threshold algorithm.
\newblock \emph{Machine Learning}, 2(4):285--318, 1988.
\newblock \href{https://ai.stanford.edu/~pabbeel/depth_qual/littlestone1988.pdf}{Article}.

\bibitem[Lu(2024)]{lu2024}
Zhou Lu.
\newblock When is inductive inference possible?
\newblock \emph{Advances in Neural Information Processing Systems},
37:92721--92744, 2024.
\newblock \href{https://proceedings.nips.cc/paper_files/paper/2024/hash/a8808b75b299d64a23255bc8d30fb786-Abstract-Conference.html}{Proceedings entry}.

\bibitem[Schulte(1999)]{schulte1999}
Oliver Schulte.
\newblock Means-ends epistemology.
\newblock \emph{The British Journal for the Philosophy of Science},
50(1):1--31, 1999.
\newblock \href{https://doi.org/10.1093/bjps/50.1.1}{doi:10.1093/bjps/50.1.1}.
\end{thebibliography}
\end{document}